\documentclass[11pt]{article}

\usepackage[final]{acl}

\usepackage{times}
\usepackage{latexsym}

\usepackage[T1]{fontenc}

\usepackage[utf8]{inputenc}

\usepackage{microtype}

\usepackage{inconsolata}

\usepackage{graphicx}

\usepackage{booktabs}
\usepackage{makecell}
\usepackage{bm}
\usepackage{bbm}
\usepackage{amsmath}
\usepackage{amssymb}
\usepackage{amsfonts}
\usepackage{amsthm}
\usepackage{mleftright}
\usepackage{thmtools}
\usepackage{thm-restate}
\usepackage[safe]{tipa}

\usepackage{graphicx}
\usepackage{todonotes}
\definecolor{mintgreen}{RGB}{152, 255, 152}
\makeatletter
\newcommand*\iftodonotes{\if@todonotes@disabled\expandafter\@secondoftwo\else\expandafter\@firstoftwo\fi}  %
\makeatother

\makeatletter
\newcommand{\appendixtableofcontents}{%
  \section*{Appendix Contents}%
  \@starttoc{atoc}%
}
\newcommand{\startappendixtoc}{%
  \let\oldaddcontentsline\addcontentsline
  \renewcommand{\addcontentsline}[3]{%
    \oldaddcontentsline{##1}{##2}{##3}%
    \ifstrequal{##1}{toc}{\oldaddcontentsline{atoc}{##2}{##3}}{}%
  }%
}
\makeatother

\usepackage{xcolor}

\newcommand{\mymacro}[1]{%
  \ifmmode
    {\textcolor{MacroColor}{#1}}%
  \else
    {\color{MacroColor}#1}%
  \fi
}
\newcommand{\defn}[1]{\textbf{#1}}
\newcommand{\defeq}{\mathrel{\stackrel{\textnormal{\tiny def}}{=}}}

\usepackage{xcolor}%
\newif\ifshowrevision          %
\showrevisiontrue              %

\ifshowrevision
  
\else
  
\fi

\usepackage{mathtools}
\usepackage{ellipsis}
\makeatletter
\renewcommand*{\mathellipsis}{%
  \mathinner{%
    \kern\ellipsisbeforegap%
    {\ldotp}\kern\ellipsisgap%
    {\ldotp}\kern\ellipsisgap%
    {\ldotp}\kern\ellipsisaftergap%
  }%
}
\renewcommand*{\dotsb@}{%
  \mathinner{%
    \kern\ellipsisbeforegap%
    {\cdotp}\kern\ellipsisgap%
    {\cdotp}\kern\ellipsisgap%
    {\cdotp}\kern\ellipsisaftergap%
  }%
}
\renewcommand*{\@cdots}{%
  \mathinner{%
    \kern\ellipsisbeforegap%
    {\cdotp}\kern\ellipsisgap%
    {\cdotp}\kern\ellipsisgap%
    {\cdotp}\kern\ellipsisaftergap%
  }%
}
\renewcommand*{\ellipsis@default}{%
  \ellipsis@before
  \kern\ellipsisbeforegap
  .\kern\ellipsisgap
  .\kern\ellipsisgap
  .\kern\ellipsisgap
  \ellipsis@after\relax}
\renewcommand*{\ellipsis@centered}{%
  \ellipsis@before
  \kern\ellipsisbeforegap
  .\kern\ellipsisgap
  .\kern\ellipsisgap
  .\kern\ellipsisaftergap
  \ellipsis@after\relax}
\AtBeginDocument{%
  \DeclareRobustCommand*{\dots}{%
    \ifmmode\@xp\mdots@\else\@xp\textellipsis\fi}}
\def\ellipsisgap{-.05em}
\def\ellipsisbeforegap{-.05em}
\def\ellipsisaftergap{.05em}
\makeatother

\usepackage[normalem]{ulem} %

\definecolor{ETHBlue}{RGB}{33,92,175}   %
\definecolor{ETHGreen}{RGB}{98,115,19}      %
\definecolor{ETHPurpleDark}{RGB}{140,10,89} %
\definecolor{ETHPurple}{RGB}{163,7,116} %
\definecolor{ETHGray}{RGB}{111,111,111} %
\definecolor{ETHRed}{RGB}{183,53,45}    %
\definecolor{ETHPetrol}{RGB}{0,120,148} %
\definecolor{ETHBronze}{RGB}{142,103,19}    %
\definecolor{ETHOrange}{RGB}{230, 100, 50}

\colorlet{MacroColor}{black}
\colorlet{MACROCOLOR}{MacroColor}

\definecolor{WildStrawberry}{rgb}{1.0, 0.26, 0.64}
\definecolor{Wisteria}{rgb}{0.79, 0.63, 0.86}
\definecolor{Olive}{rgb}{0.5, 0.5, 0.0}
\definecolor{Plum}{rgb}{0.56, 0.27, 0.52}
\definecolor{NavyBlue}{rgb}{0.0, 0.0, 0.5}
\definecolor{OliveGreen}{rgb}{0.33, 0.42, 0.18}
\definecolor{RoyalPurple}{rgb}{0.47, 0.32, 0.66}

\usepackage[capitalize,noabbrev]{cleveref}
\crefname{section}{\S}{\S\S}
\crefname{table}{Tab.}{Tab.}
\crefname{figure}{Fig.}{Figs.}
\crefname{algorithm}{Alg.}{}
\crefname{equation}{Eq.}{Eqs.}
\crefname{appendix}{App.}{Apps.}
\crefname{theorem}{Theorem}{}
\crefname{prop}{Proposition}{}
\crefname{definition}{Def.}{}
\crefname{cor}{Corollary}{}
\crefname{observation}{Observation}{}
\crefname{assumption}{Assumption}{}
\crefname{hyp}{Hyp.}{Hypotheses}
\crefformat{section}{\S#2#1#3}
\crefname{namedtheorem}{Hyp.}{Hypotheses}

\crefformat{equation}{\eqA Eq.~\eqB #2#1#3)}
\newcommand{\eqA}{}
\newcommand{\eqB}{(}

\newcommand{\saveforcameraready}[1]{}
\newcommand{\cutforspace}[1]{}

\newcommand{\justification}[1]{%
    \refstepcounter{equation}%
    \tag{\theequation \textcolor{black!50}{, \footnotesize{#1}}}
}

\usepackage{stmaryrd} %

\newcommand{\set}[1]{\left\{ #1 \right\}}

\newcommand{\indicator}[1]{\mathbbm{1}\mleft\{#1\mright\}}

\newcommand{\precover}{{\mymacro{\mathcal{P}}}}

\newcommand{\prefixbase}{\mymacro{\text{pf}}}

\newcommand{\kleene}[1]{#1^*}

\newcommand{\eos}{{\mymacro{\textsc{eos}}}}

\DeclareTextCommandDefault{\textvisiblespace}{%
  \mbox{\kern.06em\vrule \@height.3ex}%
  \vbox{\hrule \@width.3em}%
  \hbox{\vrule \@height.3ex}}

\newcommand{\X}[0]{{\mymacro{X}}}

\DeclareMathOperator*{\E}{{\mymacro{\mathbbm{E}}}}

\newcommand{\asroracle}{\mymacro{\Omega}}

\definecolor{speechcolor} {HTML}{4477AA}
\definecolor{textcolor}   {HTML}{BB5566}
\definecolor{generalcolor}{HTML}{228833}
\definecolor{gpttokcolor}  {HTML}{CCBB44}

\newcommand{\speechtok}[0]{{\color{speechcolor}\delta}}
\newcommand{\speechtoken}{{\mymacro{\speechtok}}}
\newcommand{\speechtokseq}{\ensuremath{{\mymacro{{\color{speechcolor}\boldsymbol{\delta}}}}}\xspace}
\newcommand{\speechtokAlphabet}{{\mymacro{{\color{speechcolor}\Delta}}}}
\newcommand{\speechtokStrings}[0]{{\mymacro{\kleene{{\color{speechcolor}\Delta}}}}}
 
\newcommand{\texttok}[0]{{\color{textcolor}\sigma}}
\newcommand{\texttoken}{{\mymacro{\texttok}}}
\newcommand{\texttokseq}{\ensuremath{{\mymacro{{\color{textcolor}\boldsymbol{\sigma}}}}}\xspace}
\newcommand{\texttokAlphabet}{{\mymacro{{\color{textcolor}\Sigma}}}}
\newcommand{\texttokStrings}[0]{{\mymacro{\kleene{{\color{textcolor}\Sigma}}}}}
 
\newcommand{\generaltok}[0]{{\color{generalcolor}\gamma}}
\newcommand{\generaltoken}{\mymacro{\generaltok}}
\newcommand{\generaltokseq}{\ensuremath{{\mymacro{{\color{generalcolor}\boldsymbol{\gamma}}}}}\xspace}
\newcommand{\generaltokAlphabet}{{\mymacro{{\color{generalcolor}\Gamma}}}}
\newcommand{\generaltokStrings}[0]{{\mymacro{\kleene{{\color{generalcolor}\Gamma}}}}}
 
\newcommand{\gpttok}[0]{{\color{gpttokcolor}\phi}}
\newcommand{\gpttoken}{\mymacro{\gpttok}}
\newcommand{\gpttokseq}{\ensuremath{{\mymacro{{\color{gpttokcolor}\boldsymbol{\phi}}}}}\xspace}
\newcommand{\gpttokAlphabet}{{\mymacro{{\color{gpttokcolor}\Phi}}}}

\newcommand{\LM}{\mymacro{p}}
\newcommand{\generalLM}{\mymacro{\LM_{{\color{generalcolor}\Gamma}}}}
\newcommand{\speechLM}{\mymacro{\LM_{{\color{speechcolor}\Delta}}}}
\newcommand{\textLM}{\mymacro{\LM_{{\color{textcolor}\Sigma}}}}
\newcommand{\gptLM}{\mymacro{\LM_{{\color{gpttokcolor}\Phi}}}}
\newcommand{\generalLMPref}{\mymacro{\overrightarrow{\generalLM}}}
\newcommand{\speechLMPref}{\mymacro{\overrightarrow{\speechLM}}}
\newcommand{\textLMPref}{\mymacro{\overrightarrow{\textLM}}}
\newcommand{\gptLMPref}{\mymacro{\overrightarrow{\gptLM}}}

\newcommand{\estimator}{\mymacro{\widehat{\textLMPref}}}
\newcommand{\proposal}{\mymacro{q_{\speechtokAlphabet}}}
\newcommand{\prefixProposal}{\mymacro{\overrightarrow{\proposal}}}
\newcommand{\projection}[1]{\mymacro{\pi_{#1}}}

\newcommand{\surp}{\mymacro{\iota}}
\newcommand{\surpgpt}{\mymacro{\surp_{\text{gpt}}}}
\newcommand{\surpest}{\mymacro{\surp_{\text{est}}}}
\newcommand{\surpslm}{\mymacro{\surp_{\text{slm}}}}

\newcommand{\sampsize}{\mymacro{N}}
\newcommand{\subsampsize}{\mymacro{J}}

\newcommand{\var}{\mymacro{\mathbb{V}}}

\title{Estimating Surprisal from the Speech Signal}

\newcommand{\uzurich}{\textnormal{1}}
\newcommand{\ethz}{\textnormal{2}}
\newcommand{\chifro}{\textnormal{3}}
\newcommand{\umin}{\textnormal{4}}
\author{Zachary Hopton\textsuperscript{\uzurich} \quad  Samuel Kiegeland\textsuperscript{\ethz,\,\chifro} \quad Francesco Ignazio Re\textsuperscript{\ethz} \quad Patrick Haller\textsuperscript{\uzurich} \\
  \textbf{
  Matthew B. Winn\textsuperscript{\umin} \quad 
  Ryan Cotterell\textsuperscript{\ethz} \quad
  Eleanor Chodroff\textsuperscript{\uzurich}} \\ \\
  \textsuperscript{\uzurich}University of Zurich~\;~\textsuperscript{\ethz}ETH Zürich~\;~\textsuperscript{\chifro}CHI-FRO~\;~%
  \textsuperscript{\umin}University of Minnesota\\
  \texttt{\{\href{mailto:zacharywilliam.hopton@uzh.ch}{zacharywilliam.hopton}, \href{mailto:eleanor.chodroff@uzh.ch}{eleanor.chodroff}\}}\texttt{@uzh.ch}\\
}

\date{}

\begin{document}
\maketitle
\begin{abstract}
Although language is predominantly spoken, empirical tests of surprisal theory rely almost entirely on written text. In practice, surprisal is estimated using language models (LMs) trained on text and applied to transcriptions, creating a text bottleneck that limits evaluation of the theory in the spoken modality.
We remove this bottleneck by estimating surprisal directly from the speech signal using a speech LM. This raises a key question: should surprisal be defined over realization-specific acoustic inputs or over abstract linguistic units? Since a single unit can have many distinct acoustic realizations, each assigned a different probability by the model, defining surprisal over abstract units is difficult. To address this, we introduce an importance sampling estimator that approximates the probability of an abstract unit by marginalizing over realizations generated by a text-to-speech model.
Using pupillometry data from a listening study, we evaluate surprisal theory in this text-free setting at multiple levels of granularity. Obtaining surprisal from several speech LMs, we consistently find that surprisal estimated via importance sampling better predicts pupil dilation than realization-specific surprisal for words. We also find that surprisal values for sub-lexical units are predictive of pupil dilation. While we find that surprisal can be estimated without recourse to text, and that marginalizing over distinct realizations of words is more predictive of lexical processing effort, surprisal as a whole is not a strong predictor of change in pupil size, regardless of its source.

\end{abstract}

\section{Introduction}

Human language is primarily a spoken phenomenon, yet empirical tests of surprisal theory rely on written text rather than the acoustic signal itself \citep{hale-2001-probabilistic, levy2008expectation, shain2024large}.\footnote{We note that signed languages constitute another major modality of human language production and perception \citep{sandler2006sign}. The challenges we discuss regarding the mismatch between text-based models and non-textual modalities apply equally to sign.}
In studies of reading, this assumption is natural: the input is discretized into written units, and text language models (LMs) can assign probabilities to those units directly. In studies of listening, however, the input is an acoustic signal. 
Applying a text language model to a transcription of the acoustic signal introduces a text bottleneck: the probabilistic model is defined over written forms, while listeners receive and interpret speech. 
This bottleneck is not only an inconvenience; it also obscures a fundamental question for calculating surprisal for a speech signal from a speech LM: what is the unit over which surprisal should be calculated?\looseness=-1

Existing work on spoken-language surprisal has generally followed the text-based approach by applying text LMs to transcriptions \citep{shain_fmri_2020, WallbridgeL022}. This assumes that textual representations are adequate proxies for the units used in listening. 
Even when these transcriptions are in terms of units that are more representative of speech than graphemes (e.g., phonemes or syllables as in \citet{pimentel-etal-2021-surprisal} and \citet{sohoglu2024convergent}), transcription omits prosodic, acoustic, and speaker-specific information present in the speech signal.
Speech LMs remove this bottleneck, but introduce a new problem: their probabilities are defined over speech-code sequences rather than words or phones.
\looseness=-1

We address this problem using pupillometry data from \citet{winn2016rapid}, where participants listened to sentences varying in semantic context. %
As these stimuli have minimal phonetic variation (one speaker, flat affect, read speech prosody), they provide a controlled auditory test case for asking which speech-LM surprisal quantity best predicts processing effort. %
Studying surprisal theory through the lens of speech then raises two major questions:\looseness=-1

\paragraph{Abstract units or specific acoustic realizations?}

Surprisal theory is typically defined over abstract units such as words or phonemes \citep{levy2008expectation, smith2013effect}, whereas speech LMs assign probabilities to acoustic representations.\footnote{We use the term \defn{abstract units} to refer to any given unit that might compose one's intended message when communicating, as opposed to the units that make up noisy observations of said message. See \citet{vigly2025comprehension} for relevant discussion of noisy-channel inference models of language comprehension.}
This discrepancy between the abstract and the specific is known in speech as the lack of invariance problem: a single abstract linguistic unit can have infinitely many acoustic realizations, and 
the same unit may be realized differently depending on its phonetic context \citep{liberman1967perception}, speaker characteristics \citep{peterson1952control}, recording quality, background noise, and conditions of the physical environment (e.g., reverberation, echo, etc.), among other sources of variation.
As a result, the speech LM assigns different probabilities to different realizations of the same abstract unit. 
Aggregating probabilities over the observed utterance therefore yields surprisal for that particular realization, not necessarily for the abstract unit. In the present work, we approximate abstract-unit surprisal by marginalizing over possible realizations with importance sampling \citep{robert2004monte}.\looseness=-1

\paragraph{Which granularity of speech?} Speech is continuous, but speech LMs operate over discrete frames. Surprisal can therefore be computed at the frame level, or aggregated over more conventional units such as phones or orthographic words. We calculate surprisal at each of these granularities to ask which levels of speech LM predictions can model listening effort.\looseness=-1

In our empirical analysis, we investigate the relationship between pupil dilation and speech-LM surprisal estimated at the word, phone, and frame levels using four 
speech LMs. At the word level, we further compare realization-specific surprisal against our marginalized, abstract-unit estimate. We find that marginalized word surprisal better predicts pupil dilation than realization-specific word surprisal, though phone- and frame-level surprisal are also predictive. 
Though surprisal is a weak predictor of pupil response as a whole, our results suggest that in spoken language comprehension, processing effort is better captured by probabilities over abstract lexical units than by probabilities over specific acoustic realizations alone. 
Our main contribution, however, is methodological: we provide a framework for defining the relevant probabilistic event when obtaining surprisal from speech LMs. Applying this framework to diverse speech conditions stands to position generative speech LMs as tools for learning how factors such as speech content, speaker characteristics, and prosody interact to influence speech processing.

\begin{figure}
    \centering
    \includegraphics[width=\linewidth]{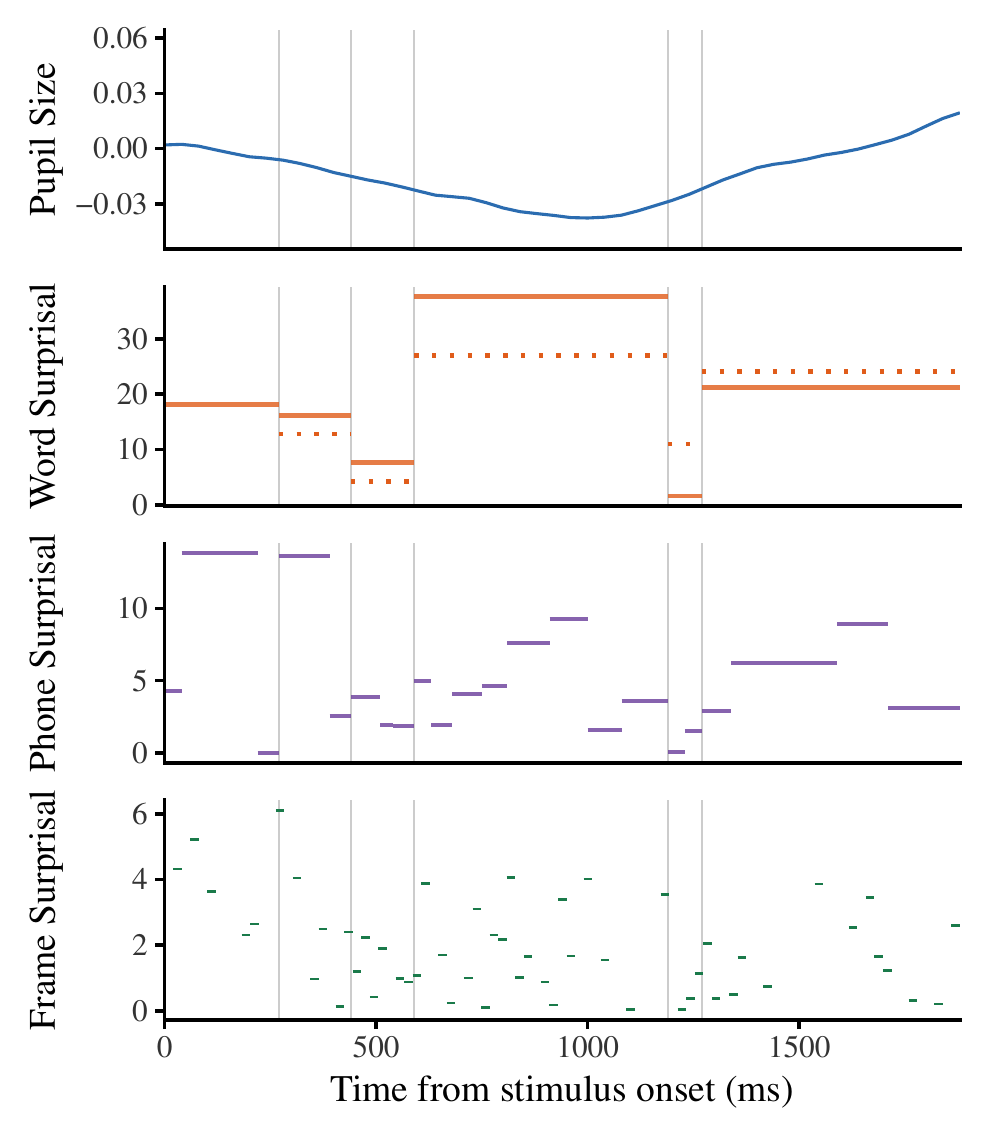}
    \caption{
    Pupil size trajectory and GSLM unit surprisals for participant NH--24 while listening to stimulus 5x39, ``Tom has been discussing the beads.'' Dotted lines represent estimated marginalized surprisals; solid lines represent unmarginalized surprisals for this specific utterance.}
    \label{fig:trial-viz}

\end{figure}

\section{Language Models}

\subsection{Preliminaries}

We denote an alphabet---some finite, non-empty set of units---using uppercase Greek letters, e.g., $\generaltokAlphabet$. We use lowercase Greek letters, e.g., $\generaltoken \in \generaltokAlphabet$, to denote units from an alphabet and bold lowercase Greek letters $\generaltokseq = \generaltoken_1 \cdots \generaltoken_T$ to refer to a string of length $T$ from $\generaltokStrings$. We use the notation $\lvert \generaltokseq \rvert$ to refer to the length of a string. The concatenation of two strings $\generaltokseq$ and $\generaltokseq'$ is denoted as $\generaltokseq\generaltokseq'$. $\generaltokseq \prec \generaltokseq'$ denotes that $\generaltokseq$ is a strict prefix of $\generaltokseq'$, and $\generaltokseq \preceq \generaltokseq'$ denotes that $\generaltokseq$ is a prefix of $\generaltokseq'$. 

A language model $\LM$ is a probability distribution over some alphabet. We denote the specific alphabet over which a language model is defined using subscripts, e.g., $\generalLM$, and extend the alphabet with a special end-of-sequence symbol $\eos \notin \generaltokAlphabet$. We use the random variable $\X$ that takes values of strings $\generaltokseq \in \generaltokStrings$ according to the distribution specified by the LM. 
The LM's \emph{prefix probability} is defined as
\begin{equation}
    \generalLMPref\left(\generaltokseq\right) 
    \;\defeq\; 
    \sum_{\generaltokseq' \in \generaltokStrings} \generalLM\left(\generaltokseq \generaltokseq'\right),
\end{equation}
the total probability mass of all sequences that begin with the prefix $\generaltokseq$.  
The conditional prefix probability is then
\begin{equation}
    \generalLMPref\left(\generaltokseq' \mid \generaltokseq\right)
    \;\defeq\;
    \frac{\generalLMPref\left(\generaltokseq \generaltokseq'\right)}{\generalLMPref\left(\generaltokseq\right)},
\end{equation}
which defines a valid probability distribution over $\generaltokAlphabet \sqcup \!\left\{\eos\right\}$ for all $\generaltokseq \in \generaltokStrings$.
This structure yields the standard left-to-right factorization of an autoregressive language model:  
\begin{equation}
\label{lr-factorization}
    \generalLM\left(\generaltokseq\right) 
    \;=\;
    \generalLMPref\left(\eos \mid \generaltokseq\right)
    \prod_{t=1}^{\lvert\generaltokseq\rvert} 
    \generalLMPref\left(\generaltoken_t \mid \generaltokseq_{<t}\right).
\end{equation}

\subsection{Speech Language Models}

Whereas text contains natural boundaries between units (e.g., spaces and punctuation in many languages' scripts), speech is carried by a continuous acoustic signal without natural boundaries. 
Digitization of the speech signal yields a dense sequence of bit-valued amplitude measurements (e.g., 16,000 samples per second).
Although these measurements are discrete in a technical sense, treating them as independent units for language modeling is poorly aligned with linguistic structure, as strong dependencies between adjacent samples reflect low-level signal properties rather than meaningful linguistic variation.
In practice, speech LMs therefore operate over higher-level representations obtained by discretizing the signal using feature extractors or neural codecs \citep{arora2025landscape}.
An input acoustic signal is first passed through a feature extractor, yielding a shorter sequence of representations. These representations can then be quantized into discrete units from a finite set $\speechtokAlphabet$. 
As described by \citet{lakhotia-etal-2021-generative}, strings of these units from $\speechtokAlphabet$ can be used to train an autoregressive language model in the same way as strings of text tokens. 

Specifically, we define $\speechtokseq$ as a string from $\speechtokStrings$. By defining a speech LM $\speechLM$ as a probability distribution over $\speechtokStrings$ extended with an end-of-sequence symbol $\eos \notin \speechtokAlphabet$, we obtain a left-to-right factorization of an autoregressive language model analogous to that in \cref{lr-factorization}:
\begin{equation}
\label{speech-lr-factorization}
    \speechLM(\speechtokseq) 
    =
    \speechLMPref\left(\eos \mid \speechtokseq\right)
    \prod_{t=1}^{\lvert\speechtokseq\rvert} 
    \speechLMPref\left(\speechtoken_t \mid \speechtokseq_{<t}\right).
\end{equation}

\subsection{Converting Speech LMs to Text LMs}

Text can be straightforwardly modeled by defining $\texttokAlphabet$ as an alphabet of units such as orthographic words. Prior work has studied the question of alphabet transfer in LMs, for example, to obtain a character interface from a token-level model \citep{vieira2025from, snbjarnarson2026transducing}. 
Similarly, we seek to model sequences of units such as orthographic words from a speech LM defined over a different alphabet, i.e., an alphabet of speech tokens.
This requires a function $\asroracle \colon \speechtokStrings\rightarrow\texttokStrings$ that decodes speech strings
to orthographic words. 
Intuitively, it could be conceptualized as a human listener or automatic speech recognition (ASR) model transcribing an utterance $\speechtokseq$. Specifically, we assume that $\asroracle$ is a \emph{prefix monotone} mapping---extending the input speech (e.g., adding a new speech token to the input sequence) can only leave the output transcription unchanged or cause it to grow; the transcription of any prior prefix is never revised. Though this is a necessary assumption for our approach to obtaining probabilities of abstract units from a speech LM (see \cref{sec:ap_is}), there are limits to the extent it realistically applies to speech processing. For example, given the input utterance ``/\textipa{aIskrim}/'', $\asroracle$ is required to commit to transcribing either ``I scream'' or ``Ice cream'', though the two could only be disambiguated with further context.

Using the mapping function $\asroracle$, words are generated by the speech LM according to the following distribution:
\begin{equation}
    \label{beta_dist}
    \textLM\left(\texttokseq\right) = \mathbb{P}_{\X\sim \speechLM}\!\left[\asroracle\left(\X\right)=\texttokseq\right].
\end{equation}

\noindent The probability of some orthographic string $\texttokseq$ is then calculated as
\begin{equation}
    \label{eq:preim_marg}
    \textLM\left(\texttokseq\right) =
    \sum_{\speechtokseq \in \precover(\texttokseq)} \indicator{\asroracle\left(\speechtokseq\right) = \texttokseq} \speechLM\left(\speechtokseq\right),
\end{equation}

\noindent
where the precover
$\precover\left(\texttokseq\right) \defeq \set{\speechtokseq \in \speechtokStrings \colon \asroracle\left(\speechtokseq\right) \succeq \texttokseq}$ is the set of all speech strings whose transcription under $\asroracle$ has $\texttokseq$ as a prefix \citep{vieira2025from}. 
Calculating \cref{eq:preim_marg} is not feasible given that we cannot readily enumerate all possible $\speechtokseq$. This being the case, we derive an estimator to approximate prefix probabilities of our units of interest (orthographic words) from a speech LM in \cref{sec:ap_is}.

\section{Estimating Unit Surprisal}
\label{sec:est}

We aim to investigate the relationship between next-unit surprisal under a speech LM and processing effort, as measured by changes in pupil size. Surprisal is computed for multiple units that vary along two dimensions: level of abstraction (abstract units versus surface realizations) and granularity of the acoustic signal (orthographic words, phones, and frames). This yields four surprisal estimates in total: abstracted word-level surprisal obtained via the estimator introduced in \cref{sec:ap_is}, and word-, phone-, and frame-level surprisal computed directly over the acoustic signal (see \cref{fig:trial-viz}). 
Below we discuss the implementation of the estimator used to obtain surprisal over abstract units.

\subsection{Importance Sampling}
\label{subsec:is}
We derive an importance sampling estimator $\estimator_{,\sampsize}\left(\texttoken_t\mid\texttokseq_{<t}\right)$ that allows us to approximate the conditional probability of an abstract orthographic word from a speech LM (see \cref{sec:ap_is}). %
In importance sampling, we do not sample directly from a speech language model (e.g., randomly sampling the next token given some context), as this would rarely result in the unit of interest $\texttoken_t$. Rather, we draw speech strings from a proposal distribution $\proposal$ and reweight them by the ratio $\speechLM  /  \proposal$ \citep{robert2004monte}. 
Due to \cref{prop:is_estimator} and linearity of expectation, the unconditional importance sampling estimator that we derive is unbiased, i.e., the expectation of the estimator equals the quantity of interest $\textLMPref\left(\texttoken\right)$. Given $\sampsize$ i.i.d.\ samples $\speechtokseq^{\left(1\right)}, \ldots, \speechtokseq^{\left(\sampsize\right)} \sim \proposal$, we first estimate the prefix probability of a text string $\texttokseq$ as
\begin{multline}
\label{eq:inc_is_est}
    \estimator_{,\sampsize}\left(\texttokseq\right) \defeq \frac{1}{\sampsize} \sum_{n=1}^{\sampsize} 
    \left[
        \indicator{\speechtokseq^{\left(n\right)} \in \precover\left(\texttokseq\right)}\right. \\
        \left.\times\frac{\speechLMPref\left(\projection{\texttokseq}\left(\speechtokseq^{\left(n\right)}\right)\right)}{\prefixProposal\left(\projection{\texttokseq}\left(\speechtokseq^{\left(n\right)}\right)\right)}
    \right],
\end{multline}

\noindent 
where 
$\projection{\texttokseq}$ projects each speech string in the precover to its shortest prefix that also lies in the precover (see \cref{sec:ap_is} for the formal derivation). The conditional probability is then estimated as the ratio
\begin{equation}
\label{eq:cond_is_est}
    \estimator_{,\sampsize}\left(\texttoken_t \mid \texttokseq_{<t}\right) \defeq \frac{\estimator_{,\sampsize}\left(\texttokseq_{<t}\texttoken_t\right)}{\estimator_{,\sampsize}\left(\texttokseq_{<t}\right)}.
\end{equation}

\noindent The conditional estimator in \cref{eq:cond_is_est} is consistent, meaning it will converge to the true conditional probability as the number of samples approaches infinity (see \cref{prop:consistency}). We analyze the estimator's variance in \cref{app:est_var} and the bias of the ratio in \cref{prop:ratio_bias}.

In our implementation of the estimator, our target distribution is an autoregressive speech LM such as the one described in \cref{speech-lr-factorization}. We test four speech LMs: a generative speech language model (GSLM) released by \citet{lakhotia-etal-2021-generative}, SpiRit-LM \citep{nguyen-etal-2025-spirit}, and the small and large variants of Llama-Mimi \citep{sugiura2025llama}. See Appendix \cref{app:models} for model details.

As our proposal distribution $\proposal$, we take a text-to-speech (TTS) model that generates speech for a given input text. Given that TTS models are trained to generate speech corresponding to the input text, we expect that sampling from $\proposal$ is more likely to yield valid realizations of our target texts than if we were to sample directly from a speech LM. We specifically use Sesame's Conversational Speech Model (CSM),\footnote{\url{https://huggingface.co/sesame/csm-1b}} which is trained on multi-speaker conversational speech. 
We choose a multi-speaker model so that the proposal covers a diverse range of voices. \cref{subsec:csm_diversity} analyzes speaker identity in the resulting samples.
CSM trains Llama-3.2-1B \citep{llama3} with an alphabet of text and audio tokens,
where audio tokens in the alphabet are derived from the first codebook of Mimi speech tokens \citep{mimi_tokenizer}. Mimi is a neural codec trained with both distillation from a semantic audio feature extractor and an autoencoding objective with a discretization bottleneck. The discretized representations at this bottleneck are the speech tokens used in CSM.  

For a given string $\texttokseq \in \texttokStrings$ with $\lvert\texttokseq\rvert=T$, we sample $\sampsize=500$ spoken realizations from the CSM TTS model that we take as our proposal distribution $\proposal$. Conditioning on the full text of the stimuli $\texttokseq$, we generate audio from CSM
using a temperature of 1.0.
We then use the Montreal Forced Aligner (MFA) \citep{McAuliffeSM0S17} to segment the sampled audio into realizations of individual orthographic words, giving us $\sampsize$ spoken realizations of each of the $T$ units in our string. Using realizations of individual units $\texttoken$ from the same sample $\speechtokseq^{(n)}$, we can then calculate $\estimator_{,\sampsize}\left(\texttoken_t\mid\texttokseq_{<t}\right)$ using \cref{eq:cond_is_est}. 

The estimator in \cref{eq:inc_is_est} requires a prefix monotone mapping function $\asroracle$ to check for precover membership (see \cref{sec:ap_is}). ASR systems that condition on the full input sequence do not satisfy this property, as additional audio can revise previously generated text. We therefore require a causal ASR system whose outputs for earlier segments remain fixed as more speech is observed.
We approximate this prefix monotone mapping function using the SimulStreaming implementation of Whisper for speech transcription \citep{Radford23, machavcek2025simultaneous}. Because our estimator requires transcriptions of all samples and we repeat the estimation for many stimuli, we use Whisper Tiny to balance speed and performance.

\subsection{Assessing Surprisal Theory}
\label{subec:surp}
Using pupillometry data as the psychometric signal of processing effort, we assess the predictive power of surprisal values obtained from our estimates of the marginalized surprisal values for orthographic words from a speech LM. We also compare these against surprisal values derived from unmarginalized orthographic word, phone, and frame probabilities, i.e., the probability the speech LM assigns to units in the specific utterances in the stimuli. 

\paragraph{Stimuli} In \citet{winn2016rapid}, proportional change in participants' pupil diameters is measured while listening to stimuli and then repeating the stimuli aloud. As in \citet{winn2016rapid}, we use recordings and transcriptions from 162 items in the Revised Speech in Noise corpus \citep{bilger1984standardization}, each of which corresponds to a single sentence. The stimuli vary in how much the context constrains the final word, for instance ranging from the very constrained phrase ``The doctor prescribes the drug'' to the more open-ended ``Bill might discuss the foam.'' We only include data from the 21 participants with normal hearing and trials with non-vocoded stimuli in our analysis. 
We describe our approaches to obtaining surprisal values for the stimuli below.

\paragraph{Estimated Marginalized Surprisal} Using the approach to importance sampling discussed in \cref{subsec:is}, we estimate the probability our speech LM assigns to abstract orthographic words in some context, $\estimator_{,\sampsize}(\texttoken_t \mid \texttokseq_{<t})$. 
We then obtain surprisal values using the following: 
\begin{equation}
    \label{eq:surprisal_est}
    \surpest\left(\texttoken_t, \texttokseq_{<t}\right) = -\log\estimator_{,\sampsize}\left(\texttoken_t\mid\texttokseq_{<t}\right).
\end{equation}

\paragraph{Unmarginalized Speech LM Surprisal} We compare the estimated marginal surprisals to unmarginalized speech LM surprisal values. For each of the stimuli from \citet{winn2016rapid}, we run the MFA on the original audio to obtain word and phone boundaries. Given these unit boundaries, we obtain conditional probabilities of a specific realization of a unit. We then calculate the surprisal for each unit in the stimuli as follows:
\begin{equation}
    \label{eq:surprisal_slm}
    \surpslm\left(\speechtoken_t,\speechtokseq_{<t}\right) = -\log\speechLMPref\left(\speechtoken_t \mid \speechtokseq_{<t}\right).
\end{equation}

\noindent We calculate the unmarginalized probabilities of the stimuli at the level of the orthographic word, the phone, and the frame. By frame, we refer to the surprisal for each speech token $\speechtoken$ in a string $\speechtokseq$, which corresponds to 20 ms of audio for GSLM and 40 ms for SpiRit-LM.

\section{Analyses}
\label{sec:analysis}

In the empirical portion of the article, we carry out two analyses. The first serves to confirm that we are sampling utterances from a diverse range of speakers to justify our choice of proposal distribution. 
Our second analysis assesses surprisal theory using surprisal from speech LMs to model pupil data.

\subsection{CSM Sample Diversity} Though the authors anecdotally note variation between the voices and prosodic features of sampled utterances from CSM, we seek to quantitatively assess whether the $\sampsize$ realizations of our stimuli indeed capture a wide variety of realizations for the same phrase. To do so, we analyze speaker embeddings of our sampled utterances from a ResNet model pretrained for speaker verification using the SpeechBrain toolkit \citep{villalba2020101026, speechbrain}. By assessing whether we have sampled from a diverse space of speakers, we aim to justify our choice of proposal distribution. 

Specifically, we randomly select one of our stimuli and obtain speaker embeddings from each of the $\sampsize$ samples of that stimulus obtained from CSM. We then randomly select a speaker from the TIMIT corpus \citep{garofolo1993darpa} and obtain speaker embeddings from each of that selected speaker's utterances in the corpus (``within-speaker'' embeddings). Finally, we randomly select one of the sentences spoken by all 462 speakers in the TIMIT corpus train split and obtain the speaker embedding for each speaker using that utterance (``between-speaker'' embeddings). 
To assess speaker diversity, we measure the cosine distance between all speaker embeddings in each of the three sets of embeddings.

\looseness=-1

\subsection{Modeling Pupil Size}
Generalized additive mixed models (GAMMs) are often used to model psycholinguistic responses given their ability to capture nonlinear relationships between predictors and response variables \citep[\emph{inter alia}]{van2019analyzing, berzak2023eye}. We therefore apply them to compare the predictive power of our surprisal measures for pupillometry data from \citet{winn2016rapid}.
\looseness=-1

Let $y_i$ denote pupil diameter at time $t_i$ for listener $\ell\left(i\right)$
and stimulus $j\left(i\right)$. Let $\omega \ge 0$ denote an assumed pupillary
lag: a pupil sample recorded at $t_i$ is taken to reflect the linguistic unit
active $\omega$ ms earlier. Let $m = m\left(i\right)$ index the unit active at
$t_i - \omega$, and write $s_m$ for its surprisal. We define $\tau_i$ as time
since onset of that unit,
$\tau_i = \left(t_i - \omega\right) - \mathrm{onset}\left(m\left(i\right)\right)$.

We run a series of models with two basic structures: a baseline model and a model with current and spillover main effects. Our baseline model consists of random effect smooths for $\ell\left(i\right)$ and $j\left(i\right)$ to account for listener and stimulus variability. To this baseline, our current and spillover main effects models add main effect smooths for surprisal of the current and $K$ prior units. We fit several main effects models to compare the various marginalized and unmarginalized speech LM surprisals defined in \cref{subec:surp}. 

The following is a formal definition of the main effects model described above:
\begin{equation}
\begin{aligned}
y_i =&  \beta_0
+ f_1\left(s_m\right)
+ \sum_{k=1}^{K} f_{1+k}\left(s_{m-k}\right)
\\ 
&+ b_{\ell\left(i\right)}
+ b_{j\left(i\right)}
+ \varepsilon_i .
\end{aligned}
\end{equation}

\noindent Here, $s_{m-k}$ denotes $k$-lag spillover over units (e.g., previous words). The functions $f_r\left(\cdot\right)$ are univariate smooth effects represented using low-rank penalized regression splines.
Random effects $b_{\ell}$ and $b_j$ account for listener- and item-level variability, and $\varepsilon_i$ represents a Gaussian noise term, $\varepsilon_i \sim \mathcal{N}\left(0,\sigma^2\right)$. 
The baseline includes only random effects and $\varepsilon_i$.\looseness=-1

Prior to model fitting, surprisal values are standardized to place predictors on a common scale. In our experiments, we set $K=1$, corresponding to single-unit spillover terms. Following \citet{winn2016rapid}, we first set $\omega=1000$ ms. However, as pupillary response is known to vary with factors like luminosity and task difficulty \citep{van2019analyzing}, we also test two other physiologically plausible offset times: $\omega \in \{900, 1300\}$. We choose 900 and 1300 ms based on discussion by \citet{van2019analyzing}.\looseness=-1

We fit models with the \texttt{bam()} function in the \texttt{mgcv} package \citep{wood_generalized_2006}. All smooth terms are estimated using restricted maximum likelihood. Models are compared using ten-fold cross-validation by item (see \cref{ap:gamm_sum} for descriptions of fold size), reporting differences in held-out log-likelihood relative to the baseline. Uncertainty is quantified via bootstrap resampling. To study whether our main effects models result in significant improvements over the baseline model in terms of predicting pupil sizes at a given time, we carry out one-sided paired permutation testing with $B=5000$ sign-flips and apply the Holm--Bonferroni procedure to correct for conducting multiple significance tests. We also use a one-sided, paired bootstrap test on the difference in held-out delta log-likelihood values to test whether marginalized surprisal values yield significantly higher delta log-likelihoods than unmarginalized surprisal values from the same speech LM. %
\section{Results}
\subsection{CSM Sample Diversity}
\label{subsec:csm_diversity}
\begin{figure}
    \centering
    \includegraphics[width=\linewidth]{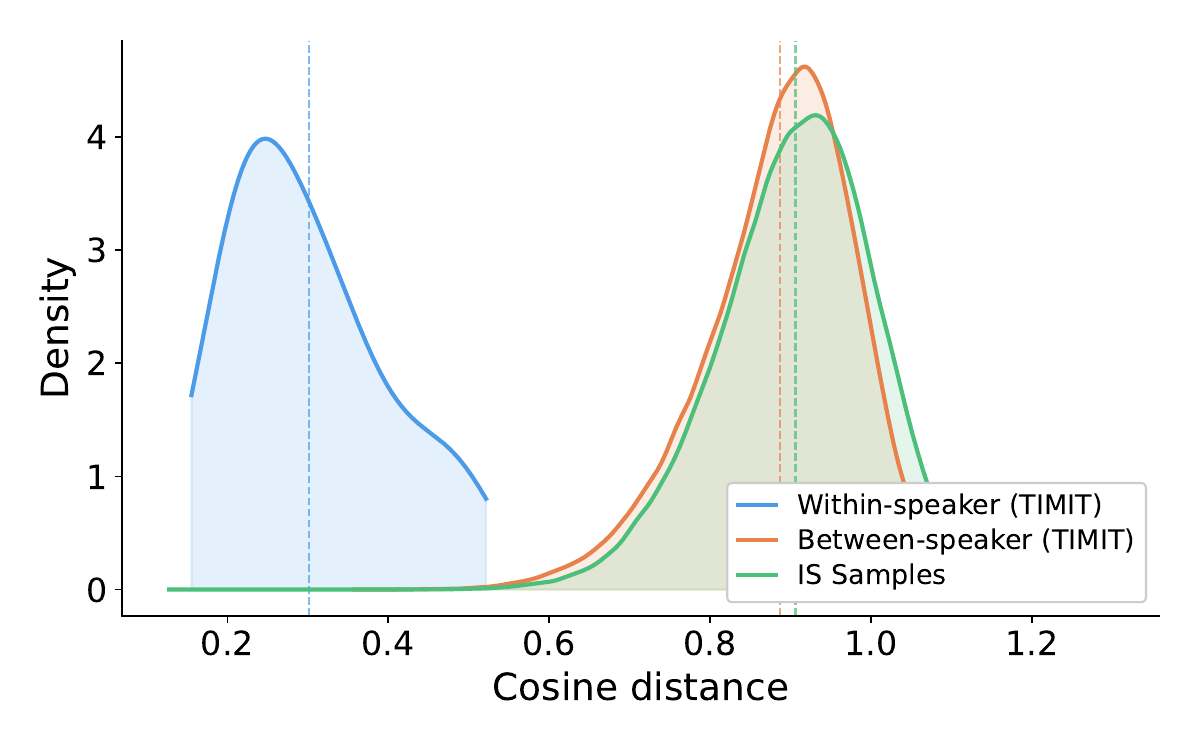}
    \caption{Distributions of cosine distances between speaker embeddings for the within-speaker TIMIT samples, between-speaker TIMIT samples, and the samples generated by CSM for the present study. Dotted vertical lines represent mean cosine distances.}
    \label{fig:speak_emb}
\end{figure}

\begin{figure*}[h]
    \centering
    \includegraphics[width=\linewidth]{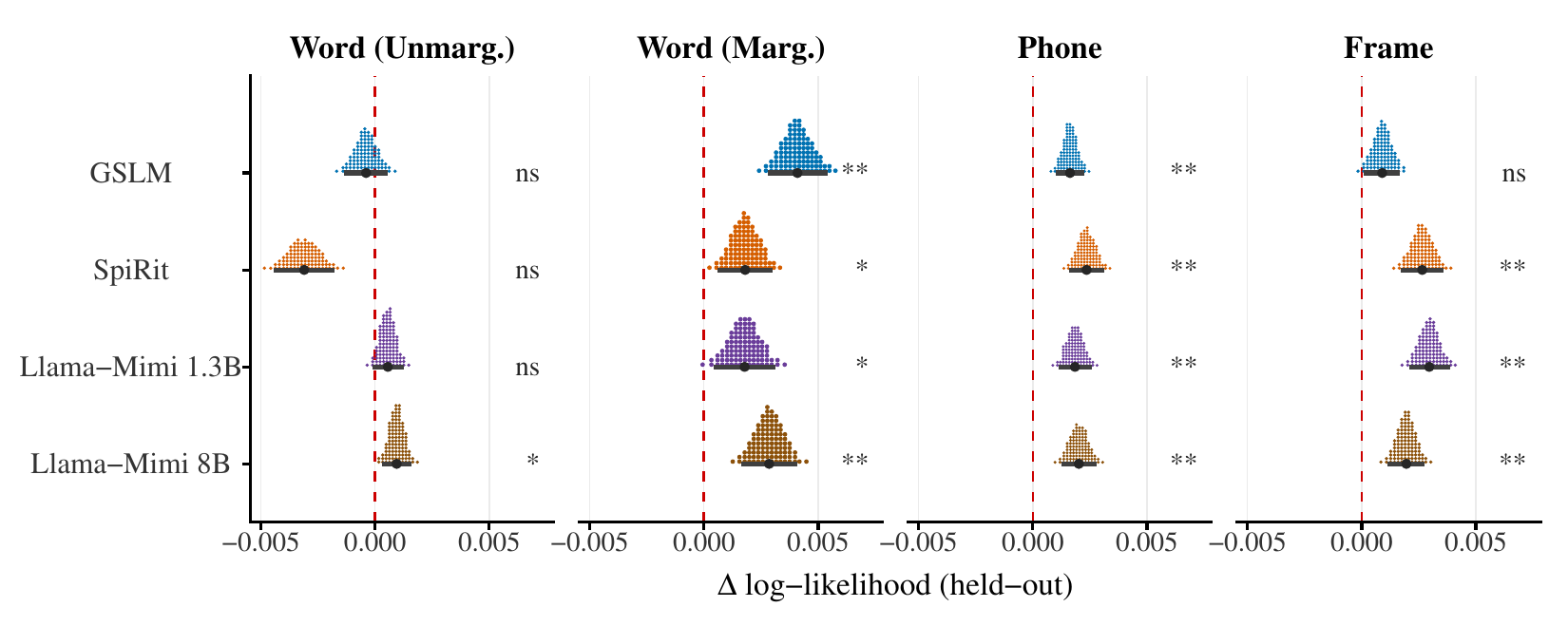}
    \caption{Held-out delta log-likelihood relative to baseline model of pupil size (offset: $\omega=1000$ ms).
    Each dot represents 1\% of the bootstrap distribution of the mean delta log-likelihood; the point and bars below give the bootstrap mean and 95\% CI. Stars correspond to significant corrected $p$-values of a model compared to the no-surprisal baseline at $\alpha$ levels of 0.05 (*), 0.01 (**), and 0.001 (***); ns = not significant.}
    \label{fig:full_dll}
\end{figure*}

When we plot the distributions of cosine distances in our sets of speaker embeddings (\cref{fig:speak_emb}), we observe that cosine distances are lower on average in our within-speaker TIMIT sample than in our between-speaker TIMIT sample. The plot shows that the distribution of cosine distances between speaker embeddings from our CSM
samples is clearly closer to the between-speaker TIMIT distribution than the within-speaker distribution.

\subsection{Modeling Pupil Size}
Delta log-likelihoods for models of pupil size fit with various surprisal values at offset of $\omega=1000$ ms are shown in \cref{fig:full_dll} (see \cref{app:sup_res} for tables of values). Note that delta log-likelihood values can only be compared within each granularity, as each granularity represents a distinct regression problem defined over a different number of pupil observations: orthographic-word windows, phone windows, and frame windows \citep{kiegeland-etal-2026-proper}. At the word level, marginalized and unmarginalized GAMMs within a speech LM are fit on the same folds (and are therefore comparable), but delta log-likelihood is not comparable between speech LMs. Because of estimator error, the ratio in \cref{eq:cond_is_est} can yield undefined surprisal values, meaning estimators for different speech LMs successfully yield estimates for different proportions of the tokens in the corpus (coverage for each speech LM is found in \cref{ap:gamm_sum}).\looseness=-1

At the level of the orthographic word, we find that for the four speech LMs tested, estimates of the marginalized surprisal values result in larger improvements over the baseline
than the unmarginalized values ($p<0.05$ for all within-model marginalized versus unmarginalized differences; see \cref{tab:marg_vs_unmarg}).  Unmarginalized lexical surprisal from speech LMs rarely results in significant improvements over the baseline, with the exception of values obtained from Llama-Mimi 8B. 
\looseness=-1

At the phone granularity, we find that 
adding unmarginalized phone surprisal consistently yields significant improvements over the baseline. At the level of the frame, unmarginalized surprisal values almost always result in significant improvements over the baseline model, except for GSLM frame surprisals. 

We present the mean delta log-likelihoods for analogous models at offsets of $\omega \in \{900,1000,1300\}$ ms in \cref{fig:offset_sweep} (see \cref{app:sup_res} for tables of values). At the word level, we observe that fitting GAMMs with marginalized surprisals consistently results in higher delta log-likelihoods than fitting GAMMs with unmarginalized lexical surprisals from the same speech LM, regardless of the offset. Improvements in delta log-likelihoods from phone and frame surprisals are generally highest at 900 and 1000 ms offsets.

For reference, we also fit models of pupil size using word-level surprisals obtained from the text LM GPT-2 \citep{radford2019language} over the stimuli transcriptions (see \cref{app:gpt2} for details and results). These also result in significant improvements over the baseline with no surprisal. The larger the offset used, the less predictive GPT-2-derived surprisals are of pupil size. We note that comparison with word-level delta log-likelihoods from speech LMs is not warranted. Regardless of the specific delta log-likelihoods obtained, our approach derives the target predictor (surprisal) from the same input signal as participants in the pupillometry study, making it a more valid approach to obtaining surprisal in a listening comprehension study than using the transcribed speech as the input signal. 

Finally, we refit our models, now adding a smooth term for the item-normalized time at which pupil dilation is recorded. Adding this term to both the main effects models and the baseline allows us to consider the extent to which surprisal adds predictive power to pupil dilation beyond the time in the sequence. The results indicate that adding surprisal to the baseline of time and random effects for participants and items does not result in significant improvements, regardless of the modality or granularity of the unit.

\section{Discussion}

\begin{figure*}[t]
    \centering
    \includegraphics[width=\linewidth]{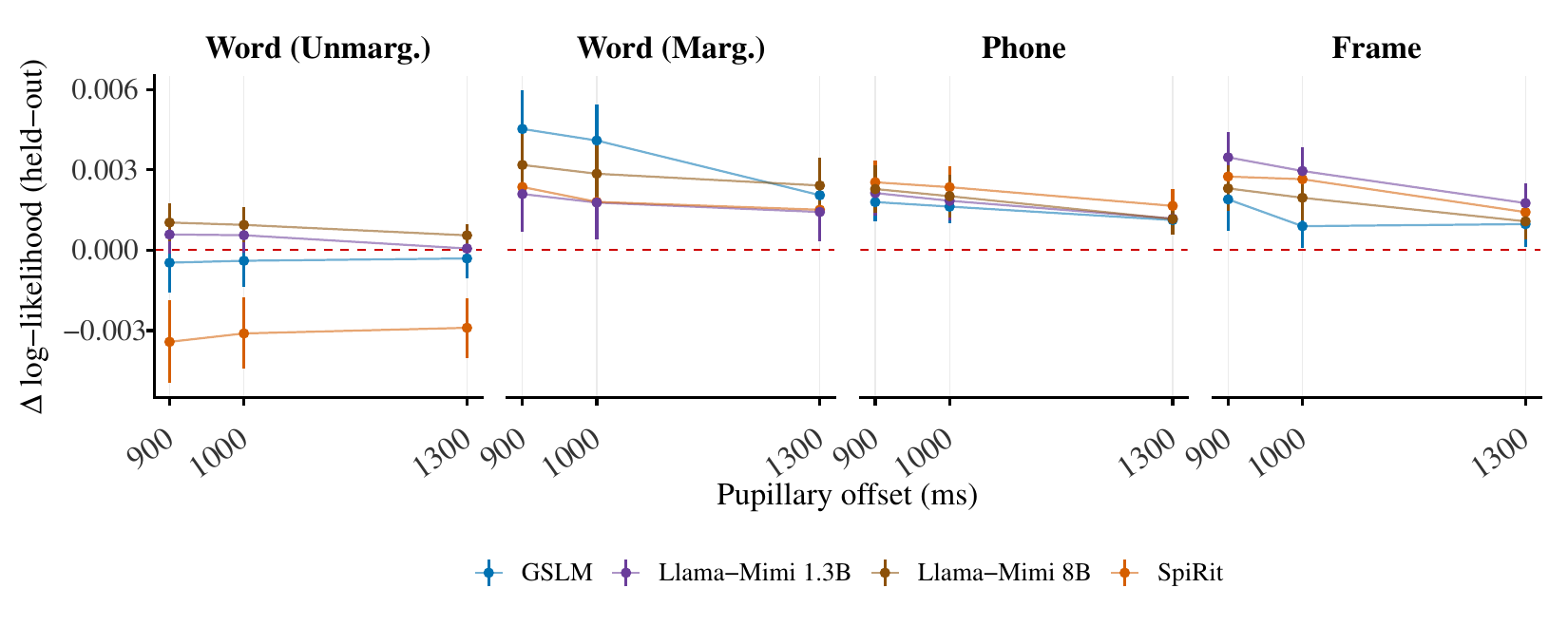}
    \caption{Held-out delta log-likelihood relative to baseline model of pupil size at various offset times. Points give the mean delta log-likelihood and bars give 95\% CI. Values at 1000 ms reflect those in \cref{fig:full_dll}. See \cref{app:sup_res} for significance of delta log-likelihood values over the baseline model.}
    \label{fig:offset_sweep}
\end{figure*}

In the present work, we asked how surprisal should be defined when it is estimated from the acoustic signal using a speech LM rather than from text. 
Because speech LMs assign probabilities to speech-code sequences, their probabilities can be interpreted either as probabilities of specific acoustic realizations or, after marginalizing over possible realizations, as probabilities of abstract linguistic units. 
Our results favor the latter interpretation at the word level: marginalized surprisal estimates for abstract lexical units were more predictive of pupil dilation than realization-specific word surprisal. At the same time, surprisal computed over finer-grained units (phones and frames) also improved models of pupil size, suggesting that speech-LM probabilities can model information relevant to listening effort at multiple granularities.

Taken together, our results suggest that in speech processing, humans are able to abstract over varying acoustic characteristics of input and use abstract lexical units to make inferences about upcoming words. This being the case, models of speech comprehension should distinguish between the probability of a particular acoustic realization and the probability of the abstract lexical unit it realizes.
This is consistent with the speech perception literature: ERP data collected while participants listened to stories indicate that listeners anticipate next words given a context \citep{van_berkum_anticipating_2005}, and word recognition in noisy conditions is easier when the word is semantically or grammatically predictable in context \citep{miller_perceptual_1963}. Moreover, there is evidence from pupillometry studies that with more context, repairing misheard words requires less cognitive effort \citep{smith2025repairing}. Our findings also parallel work on reading comprehension, for example the finding by \citet{vigly2025comprehension} that surprisal of surface forms is no longer predictive of reading times in texts with typographical errors. \looseness=-1

Our findings for phones and frames are consistent with surprisal theory research that has previously postulated that sub-lexical units are also processed in a predictive manner \citep{smith2013effect}. In certain contexts, it has been shown that finer-grained units may play a role in probabilistic anticipation during speech processing. For example, eye-tracking data indicates that variation in voice onset time, a primary correlate to the stop voicing contrast (e.g., /p/ vs. /b/) may be used to predict an upcoming word in ambiguous contexts \citep{mcmurray_within-category_2009}. Given that predictions from cues like voice onset time are based on even finer-grained input than the phone level, such anticipatory processing may explain the significance of our frame-level surprisal models, despite these units not being linguistically motivated. \looseness=-1

Our findings show that surprisal from speech LMs varies in how well it predicts pupil dilation depending on granularity: for example, whereas unmarginalized phone surprisals from GSLM result in significant improvements over the baseline model, GSLM's unmarginalized frame and word surprisals do not. While there is a large body of work studying how characteristics such as size and training data impact the so-called psychometric predictive power (PPP) of text LMs \citep[\emph{inter alia}]{kuribayashi-etal-2024-psychometric, oh-etal-2025-inverse}, this line of work has yet to be explored for speech LMs. 
Modeling decisions that yield high PPP for some units may not necessarily do the same for other units, so future work should study how different speech LM characteristics impact their cognitive plausibility across unit inventories.
\looseness=-1

Our finding that surprisal---regardless of modality, abstraction, or unit over which it is obtained---does not add predictive value over a baseline that includes information about when a pupil measurement was taken is consistent with a growing body of literature about the relative importance of surprisal as a predictor in sentence processing. Indeed, it has been shown that factors such as unigram frequency, word length, oculomotor noise, and parafoveal preview can all explain more variance in reading times than surprisal alone \citep{heilbron2023lexical,opedal-etal-2024-role}. Thus, while we provide an approach to estimating word probabilities from speech LMs, our empirical results add to the body of work suggesting a relatively less important role for surprisal in sentence processing.

Our results also have implications for natural language processing. Notably, they call into question the construct validity of common speech language modeling benchmarks such as sBLIMP and sWUGGY \citep{nguyen2020zero}. In these benchmarks, recordings of words and non-words (sWUGGY) or syntactically correct and incorrect (sBLIMP) sentences are given to speech language models, and higher scores are given when models deem the words and grammatical audio more probable than the non-words and ungrammatical audio. That is, these benchmarks score models based on the probability they assign to individual realizations of words or sentences. As we have argued here, the probabilities assigned by speech LMs to specific utterances are not representative of the probabilities the model assigns to abstract linguistic units, as the probabilities can reflect the specific utterance's prosodic, acoustic, or speaker characteristics. This is particularly concerning for cases where the benchmark has been applied to evaluate speech LMs that presume to model abstract, invariant units \citep[\textit{inter alia}]{poli-spidr, chang25_interspeech}. We hope our work encourages further benchmark development for speech language modeling that accounts for the idea that the relative probability assigned to a specific utterance might change if we manipulate the prosody, voice, or other properties of the audio. \looseness=-1

\section{Conclusion}

The current body of literature on surprisal theory applies mainly to questions of sentence processing during reading rather than spoken language comprehension. Though speech language models provide a means of applying a similar research program to spoken language comprehension, obtaining surprisal from the speech signal brings several added challenges: namely, the specific psychometric signal to model, the question of which units to consider given the naturally continuous nature of speech, and the complication of modeling many variations of the same abstract unit. We address the last of these with an importance sampling estimator that approximates the surprisal a speech LM assigns to abstract units. 
Among the approaches to obtaining surprisal from a speech LM that were tested, we found that our estimations of abstract, lexical surprisals better modeled our signal of processing effort, pupil size, than realization-specific word probabilities. We also found that phone and frame surprisal from speech LMs improved models of pupil size. 
These findings suggest that the surprisal of abstract lexical units is a key component of processing difficulty in spoken language comprehension.  Our approach to estimating surprisal for abstract units opens the door to further study the contexts in which we base our predictions of upcoming input on sub-lexical units and specific realizations. 

\section*{Limitations}
Our experiments focus on a single listening study \citep{winn2016rapid} with 162 items from the Revised Speech in Noise corpus \citep{bilger1984standardization}. We only analyze data from normal hearing listeners, but further examination would be needed to draw conclusions about whether the results hold for people with hearing loss, who have been shown to use context differently than normal hearing listeners \citep{winn2016rapid}. The nature of the stimuli limits the extent to which we can study sub-lexical phenomena and the impact of prosody and indexical information in the speech signal on processing effort during listening. In stimuli designed to study listening in unfamiliar stress patterns (e.g., accented stress patterns), the relative predictive power of marginalized and unmarginalized surprisals from a speech LM may differ. 
Likewise, we could imagine the following stimuli in a listening study: ``I reached into my toolbox and grabbed a hamburger.'' The word \textit{hamburger} is surprising in this context, but because the phonetic sequence \textipa{/hæm-/} is expected (e.g., as the start of the word ``hammer''), phone surprisal might also be predictive of pupil size. That being the case, without stimuli designed to capture sub-lexical effects, we likely do not fully characterize the role of finer-grained units' surprisal during speech processing. We are also limited to read (rather than spontaneous) speech in English in this study.\looseness=-1

Another limitation of the present study is the ASR sample rejection; though this step is necessary to exclude TTS errors, it is also a point at which errors can be made. Though a fuller error analysis is not possible because we do not have ground truth transcriptions for the sampled audio, we suspect that the most likely errors would be mis-transcriptions of valid realizations of a word. This may result in a bias against audio with especially atypical prosody or speaker features. We expect that exploring larger ASR models would decrease such errors, albeit at higher computational costs. We expect that better ASR models and a higher initial sample size $N$ would also contribute to higher estimator coverage overall.\looseness=-1

We also note that the present work only studies abstract units at the level of orthographic words and does not estimate marginalized phone surprisals. A direct application of our importance sampling approach to the phone granularity would yield estimates of the probabilities speech LMs assign to maximally context-sensitive phones in our stimuli across many speakers. However, extending our approach to study less context-sensitive units at the phone level (e.g., allophones or phonemes) would require extending the method to also sample the relevant units in diverse phonological contexts.\looseness=-1

\section*{Ethical Concerns}

We use only previously studied pupillometry and speech data with explicit permission from the author \citep{winn2016rapid}. We do not foresee any ethical concerns from the use of this data to study human speech processing.

\section*{Acknowledgments}

ZH was supported by the Swiss National Science
Foundation (Grant No. 10003607).

\bibliography{custom}

\newpage
\onecolumn

\appendix
\startappendixtoc

\appendixtableofcontents
\clearpage

\section{Derivation of the Importance Sampling Estimator}
\label{sec:ap_is}

\begin{restatable}{proposition}{isEstimator}
\label{prop:is_estimator}
Let $\speechLM \ll \proposal$ (i.e., $\speechLM$ is absolutely continuous with respect to $\proposal$),
and let $\asroracle \colon \speechtokStrings \to \texttokStrings$ be a prefix monotone function.
Define the precover $\precover\left(\texttokseq\right) \defeq \left\{\speechtokseq \in \speechtokStrings \colon \asroracle\left(\speechtokseq\right) \succeq \texttokseq\right\}$.
Let $\projection{\texttokseq} \colon \precover\left(\texttokseq\right) \rightarrow \prefixbase\left(\precover\left(\texttokseq\right)\right)$ project each $\speechtokseq$ to its longest prefix in $\prefixbase\left(\precover\left(\texttokseq\right)\right)$.\footnote{The prefix base of the precover is defined as $\prefixbase\left(\precover\left(\texttokseq\right)\right) \defeq \!\left\{\speechtokseq \in \precover\left(\texttokseq\right) \colon \nexists\, \speechtokseq' \in \precover\left(\texttokseq\right) \text{ s.t.\ } \speechtokseq' \prec \speechtokseq\right\}$.}
Then, for any $\texttokseq \in \texttokStrings$, we have
\begin{equation}
\label{eq:is_prefix}
\textLMPref\left(\texttokseq\right) = \E_{\speechtokseq \sim \proposal} \indicator{\speechtokseq \in \precover\left(\texttokseq\right)}\frac{\speechLMPref\left(\projection{\texttokseq}\left(\speechtokseq\right)\right)}{\prefixProposal\left(\projection{\texttokseq}\left(\speechtokseq\right)\right)}
\end{equation}
\end{restatable}

\begin{proof}
\begin{subequations}
\begin{align}
    \textLMPref(\texttokseq) &= \sum_{\speechtokseq \in \speechtokStrings} \indicator{\asroracle(\speechtokseq)\succeq\texttokseq}\speechLM(\speechtokseq) \justification{def.\ of $\textLMPref$} \\
     &= \sum_{\speechtokseq \in \precover(\texttokseq)} \speechLM(\speechtokseq) \justification{def.\ of $\precover(\texttokseq)$} \\
 &= \sum_{\speechtokseq \in \prefixbase(\precover(\texttokseq))} \speechLMPref(\speechtokseq) \justification{$\precover(\texttokseq)$ decomposes into disjoint cylinders} \\
  &= \sum_{\speechtokseq \in \prefixbase(\precover(\texttokseq))} \prefixProposal(\speechtokseq)\frac{\speechLMPref(\speechtokseq)}{\prefixProposal(\speechtokseq)} \justification{multiply and divide by $\prefixProposal$; valid since $\speechLM \ll \proposal$}  \\
  &= \sum_{\speechtokseq \in \speechtokStrings}\proposal(\speechtokseq)  \indicator{\speechtokseq \in \precover(\texttokseq)} \frac{\speechLMPref(\projection{\texttokseq}(\speechtokseq))}{\prefixProposal(\projection{\texttokseq}(\speechtokseq))} \justification{$\sum_{\speechtokseq' \succeq \speechtokseq} \proposal(\speechtokseq') = \prefixProposal(\speechtokseq)$; reindex via $\projection{\texttokseq}$}  \\
    &= \E_{\speechtokseq \sim \proposal} \indicator{\speechtokseq \in \precover(\texttokseq)}\frac{\speechLMPref(\projection{\texttokseq}(\speechtokseq))}{\prefixProposal(\projection{\texttokseq}(\speechtokseq))}.  \justification{def.\ of expectation under $\proposal$}
\end{align}
\end{subequations}
\end{proof}
Given $\sampsize$ i.i.d.\ samples $\speechtokseq^{(1)}, \ldots, \speechtokseq^{(\sampsize)} \sim \proposal$, the Monte Carlo estimator corresponding to \cref{eq:is_prefix} is
\begin{equation}
\label{eq:mc_estimator}
    \estimator_{,\sampsize}(\texttokseq) \defeq \frac{1}{\sampsize} \sum_{n=1}^{\sampsize} \indicator{\speechtokseq^{(n)} \in \precover(\texttokseq)}\frac{\speechLMPref(\projection{\texttokseq}(\speechtokseq^{(n)}))}{\prefixProposal(\projection{\texttokseq}(\speechtokseq^{(n)}))}.
\end{equation}

The conditional prefix probability is then estimated as the ratio
\begin{equation}
\label{eq:cond_mc_estimator}
    \estimator_{,\sampsize}(\texttoken_t \mid \texttokseq_{<t}) \defeq \frac{\estimator_{,\sampsize}(\texttokseq_{<t}\texttoken_t)}{\estimator_{,\sampsize}(\texttokseq_{<t})}.
\end{equation}
\begin{restatable}{proposition}{consistency}
\label{prop:consistency}
Let $\speechtokseq^{(1)}, \ldots \speechtokseq^{(N)} \overset{\text{i.i.d.}}{\sim} \proposal$ and suppose that
\begin{equation}\label{eq:finite-second-moment}
    \E_{\speechtokseq \sim \proposal}\left[\left(\indicator{\speechtokseq \in \precover(\texttokseq)}\frac{\speechLMPref(\projection{\texttokseq}(\speechtokseq))}{\prefixProposal(\projection{\texttokseq}(\speechtokseq))}\right)^2\right] < \infty
\end{equation}
for all $\texttokseq \in \texttokStrings$ and that $\textLMPref(\texttokseq) > 0$ for all $\texttokseq \in \texttokStrings$. Then, we have that
\begin{equation}
    \estimator_{,\sampsize}(\texttoken_t \mid \texttokseq_{<t}) \xrightarrow[\sampsize \to \infty]{\text{a.s.}} \textLMPref(\texttoken_t \mid \texttokseq_{<t}).
\end{equation}
\end{restatable}

\begin{proof}
The integrand in \cref{eq:mc_estimator} has finite second moment by assumption. 
By the strong law of large numbers, $\estimator_{,\sampsize}(\texttokseq) \xrightarrow[\sampsize \to \infty]{\text{a.s.}} \textLMPref(\texttokseq)$ for all $\texttokseq \in \texttokStrings$.
Because $\textLMPref(\texttokseq_{<t}) > 0$, the function $f(x, y) = x/y$ is continuous at $(\textLMPref(\texttokseq_{<t}\texttoken_t),\, \textLMPref(\texttokseq_{<t}))$. By the continuous mapping theorem, we arrive at
\begin{equation}
\estimator_{,\sampsize}(\texttoken_t \mid \texttokseq_{<t}) = f(\estimator_{,\sampsize}(\texttokseq_{<t}\texttoken_t),\, \estimator_{,\sampsize}(\texttokseq_{<t})) \xrightarrow[\sampsize \to \infty]{\text{a.s.}} f(\textLMPref(\texttokseq_{<t}\texttoken_t),\, \textLMPref(\texttokseq_{<t})) = \textLMPref(\texttoken_t \mid \texttokseq_{<t}).
\end{equation}
\end{proof}

\begin{restatable}{proposition}{ratioBias}
\label{prop:ratio_bias}
Adopt the assumptions of \cref{prop:is_estimator}, and suppose
$\textLMPref(\texttokseq_{<t}) > 0$. For $\speechtokseq \sim \proposal$ define
the single-sample importance weights
\begin{align*}
A &\defeq \indicator{\speechtokseq \in \precover(\texttokseq_{<t}\texttoken_t)}\frac{\speechLMPref(\projection{\texttokseq_{<t}\texttoken_t}(\speechtokseq))}{\prefixProposal(\projection{\texttokseq_{<t}\texttoken_t}(\speechtokseq))}, \\
B &\defeq \indicator{\speechtokseq \in \precover(\texttokseq_{<t})}\frac{\speechLMPref(\projection{\texttokseq_{<t}}(\speechtokseq))}{\prefixProposal(\projection{\texttokseq_{<t}}(\speechtokseq))},
\end{align*}
and suppose $\E[A^2] < \infty$ and $\E[B^2] < \infty$. Assume further that the
denominator estimator $\estimator_{,\sampsize}(\texttokseq_{<t})$ is almost
surely positive and that the conditional estimator from
\cref{eq:cond_mc_estimator} has finite variance. Writing its bias as
\begin{equation*}
\mathrm{Bias}_{\sampsize} \defeq \E\!\left[\estimator_{,\sampsize}(\texttoken_t \mid \texttokseq_{<t})\right] - \textLMPref(\texttoken_t \mid \texttokseq_{<t}),
\end{equation*}
the following hold.
\begin{enumerate}
\item[(i)] \emph{(Exact bias identity.)}
\begin{equation}
\label{eq:ratio_bias_identity}
\mathrm{Bias}_{\sampsize} = -\,\frac{\mathrm{Cov}\left(\estimator_{,\sampsize}(\texttoken_t \mid \texttokseq_{<t}),\; \estimator_{,\sampsize}(\texttokseq_{<t})\right)}{\textLMPref(\texttokseq_{<t})}.
\end{equation}
\item[(ii)] \emph{(Relative-bias bound.)}
\begin{equation}
\label{eq:ratio_bias_bound}
\frac{\left|\mathrm{Bias}_{\sampsize}\right|}{\sqrt{\var\left(\estimator_{,\sampsize}(\texttoken_t \mid \texttokseq_{<t})\right)}}
\;\le\;
\frac{\sqrt{\var\left(\estimator_{,\sampsize}(\texttokseq_{<t})\right)}}{\textLMPref(\texttokseq_{<t})}
\;=\;
\frac{1}{\sqrt{\sampsize}}\,\frac{\sqrt{\var(B)}}{\textLMPref(\texttokseq_{<t})}.
\end{equation}
\end{enumerate}
In particular, the bias measured relative to the estimator's own standard error
is $O(\sampsize^{-1/2})$: it shrinks strictly faster than the sampling noise and
is thus asymptotically negligible relative to it.
\end{restatable}
 
\begin{proof}
Write $R \defeq \estimator_{,\sampsize}(\texttoken_t \mid \texttokseq_{<t}) = \estimator_{,\sampsize}(\texttokseq_{<t}\texttoken_t)\,/\,\estimator_{,\sampsize}(\texttokseq_{<t})$
and $r \defeq \textLMPref(\texttoken_t \mid \texttokseq_{<t}) = \textLMPref(\texttokseq_{<t}\texttoken_t)/\textLMPref(\texttokseq_{<t})$.
Each of $\estimator_{,\sampsize}(\texttokseq_{<t}\texttoken_t)$ and
$\estimator_{,\sampsize}(\texttokseq_{<t})$ is a sample mean of $\sampsize$
i.i.d.\ copies of $A$ and $B$ respectively (cf.\ \cref{eq:mc_estimator}), so by
\cref{prop:is_estimator},
$\E[\estimator_{,\sampsize}(\texttokseq_{<t}\texttoken_t)] = \textLMPref(\texttokseq_{<t}\texttoken_t)$
and $\E[\estimator_{,\sampsize}(\texttokseq_{<t})] = \textLMPref(\texttokseq_{<t})$.
 
\emph{Part (i).} Because $\estimator_{,\sampsize}(\texttokseq_{<t}) > 0$ almost
surely, $R \cdot \estimator_{,\sampsize}(\texttokseq_{<t}) = \estimator_{,\sampsize}(\texttokseq_{<t}\texttoken_t)$
identically, and therefore
\begin{align}
\mathrm{Cov}\left(R,\, \estimator_{,\sampsize}(\texttokseq_{<t})\right)
&= \E\!\left[R\,\estimator_{,\sampsize}(\texttokseq_{<t})\right] - \E[R]\,\E\!\left[\estimator_{,\sampsize}(\texttokseq_{<t})\right] \notag \\
&= \E\!\left[\estimator_{,\sampsize}(\texttokseq_{<t}\texttoken_t)\right] - \E[R]\,\textLMPref(\texttokseq_{<t}) \notag \\
&= \textLMPref(\texttokseq_{<t}\texttoken_t) - \E[R]\,\textLMPref(\texttokseq_{<t}). \label{eq:cov-step}
\end{align}
Dividing \cref{eq:cov-step} by $\textLMPref(\texttokseq_{<t}) > 0$ and rearranging,
\begin{equation*}
\E[R] - \underbrace{\frac{\textLMPref(\texttokseq_{<t}\texttoken_t)}{\textLMPref(\texttokseq_{<t})}}_{=\,r}
= -\,\frac{\mathrm{Cov}\left(R,\, \estimator_{,\sampsize}(\texttokseq_{<t})\right)}{\textLMPref(\texttokseq_{<t})},
\end{equation*}
which is \cref{eq:ratio_bias_identity}.
 
\emph{Part (ii).} By the Cauchy--Schwarz inequality,
$\left|\mathrm{Cov}\left(R,\, \estimator_{,\sampsize}(\texttokseq_{<t})\right)\right| \le \sqrt{\var(R)\,\var\left(\estimator_{,\sampsize}(\texttokseq_{<t})\right)}$.
Substituting into \cref{eq:ratio_bias_identity} and dividing by $\sqrt{\var(R)}$
gives the bound in \cref{eq:ratio_bias_bound}; the final equality uses
$\var\left(\estimator_{,\sampsize}(\texttokseq_{<t})\right) = \var(B)/\sampsize$,
as $\estimator_{,\sampsize}(\texttokseq_{<t})$ averages $\sampsize$ i.i.d.\ copies
of $B$. Since the right-hand side of \cref{eq:ratio_bias_bound} is
$O(\sampsize^{-1/2})$, the relative bias vanishes at rate $\sampsize^{-1/2}$. The inequality in \cref{eq:ratio_bias_bound} is the classical relative-bias bound for a
ratio estimator \citep{hartley1954unbiased}.
\end{proof}

\section{Estimator Variance}
\label{app:est_var}

\begin{figure}
    \centering
    \small
    \includegraphics[width=0.65\linewidth]{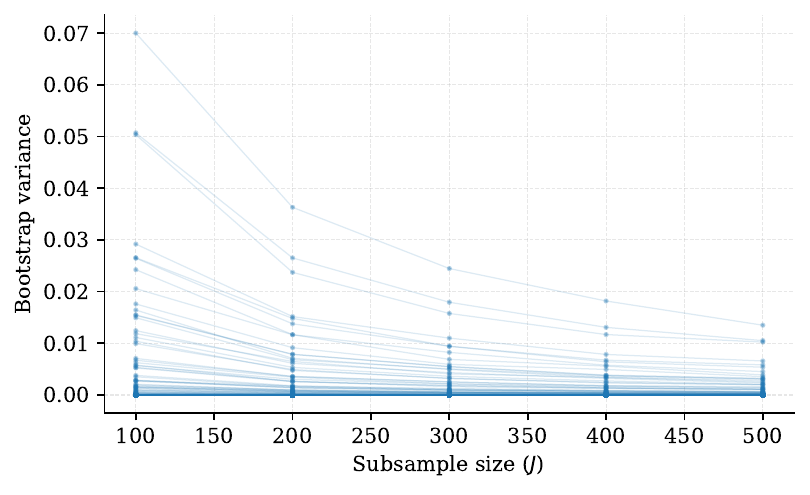}
    \caption{Variance of the IS estimator estimated via nonparametric bootstrap (1000 resamples per subsample size). Target speech LM is Llama-Mimi 1.3B. Each point shows the bootstrap variance for a given word at a given subsample size, and lines connect the same word across sample sizes. The top 2\% of variances are excluded for visualization.}
    \label{fig:var}
\end{figure}

As a diagnostic of the stability of our implementation of the estimator in \cref{eq:inc_is_est}, we analyze its sampling variance. This is akin to asking how much our estimate would change if we drew a different set of samples from CSM. If we are sampling from the full possible acoustic space of an abstracted unit, the specific set of samples should not impact the final estimate,  meaning the variance should be low when we resample. We assess this by repeatedly resampling from our data and measuring how much the estimates change. For our estimator $\estimator_{,\sampsize}$
computed from $\sampsize$ samples, the sampling variance is defined as

\begin{equation}
\var\left(\estimator_{,\sampsize}\right) = \E\!\left[(\estimator_{,\sampsize} - \E\!\left[\estimator_{,\sampsize}\right])^2\right]. 
\end{equation}
Because generating independent Monte Carlo runs is computationally expensive, we obtain a single set of $500$ importance samples per orthographic word in the \citet{winn2016rapid} stimuli. To approximate the sampling variance at different effective sample sizes, we employ nonparametric bootstrap resampling. 

For each word and for each subsample $\subsampsize \in \{100, 200, 300, 400, 500\}$, we draw 1,000 bootstrap samples (with replacement) of size $\subsampsize$ from the original 500 importance weights and recompute the estimator. 
The variance across these bootstrap replicates provides an estimate of $\var\left(\estimator_{,\subsampsize}\right)$.
We find that across orthographic words in our stimuli, our estimator yields lower variance estimates of the word's probability with larger sample sizes (see \cref{fig:var}). That is, as the sample size used to estimate a word's probability increases, the estimation becomes more robust to the specific sample used.

\section{Speech LM Details}
\label{app:models}

Here we give further details on the speech LMs we use as target distributions $\speechLM$ to obtain surprisal values from the speech signal. The generative speech language model (GSLM) released by \citet{lakhotia-etal-2021-generative} is a transformer decoder trained with a context length of 3,072 speech tokens. The speech tokens are drawn from a 200-unit codebook obtained by applying $k$-means clustering to features from layer 6 of the HuBERT feature extractor \citep{hsuetal-2021-hubert}. Due to the temporal resolution of HuBERT, each speech token corresponds to 20 ms of audio sampled at 16 kHz. The context length of 3,072 tokens, therefore, corresponds to 61.44 s of audio. We refer to these tokens as representations for \emph{frames} of speech.

SpiRit-LM is a continuously pretrained, multimodal variant of Llama-2 \citep{nguyen-etal-2025-spirit, touvron2023llama2openfoundation}. The 7B-parameter model was further pretrained on strings of interleaved text and speech tokens, meaning it can take speech as input and generate it. Its speech tokens are also obtained from discretized features of a HuBERT variant, with each speech token corresponding to 40 ms of speech sampled at 16 kHz.

Llama-Mimi trains Llama-3 \citep{llama3} for speech language modeling using representations from Mimi \citep{mimi_tokenizer}, the same neural codec used by our proposal distribution. While Llama-Mimi extends Llama's vocabulary, it does not interleave text and speech during further pretraining. We use both the 1.3B and 8B variants of Llama-Mimi, which are both trained on 240k hours of speech data. All models are trained and evaluated on English speech.

\section{GAMM Data Summary}
\label{ap:gamm_sum}

\begin{table}[h]
\centering
\small
\setlength{\tabcolsep}{6pt}
\begin{tabular}{lcrrr}
\toprule
Condition & Coverage & $N$ obs. & Train & Test \\
\midrule
\multicolumn{5}{l}{\textit{Word-level}} \\
\midrule
GPT-2 & -- & $29{,}150$ & $26{,}235$ & $2{,}915$ \\
GSLM (Unmarg.) & -- & $19{,}418$ & $17{,}476$ & $1{,}942$ \\
GSLM (Marg.) & $959/1{,}229$ (78.0\%) & $19{,}418$ & $17{,}476$ & $1{,}942$ \\
SpiRit (Unmarg.) & -- & $10{,}841$ & $9{,}757$ & $1{,}084$ \\
SpiRit (Marg.) & $809/1{,}229$ (65.8\%) & $10{,}841$ & $9{,}757$ & $1{,}084$ \\
Llama-Mimi (Unmarg.) & -- & $22{,}652$ & $20{,}387$ & $2{,}265$ \\
Llama-Mimi (Marg.) & $1{,}010/1{,}229$ (82.2\%) & $22{,}652$ & $20{,}387$ & $2{,}265$ \\
Llama-Mimi-8B (Unmarg.) & -- & $22{,}515$ & $20{,}264$ & $2{,}252$ \\
Llama-Mimi-8B (Marg.) & $1{,}009/1{,}229$ (82.1\%) & $22{,}515$ & $20{,}264$ & $2{,}252$ \\
\addlinespace
\multicolumn{5}{l}{\textit{Phone-level}} \\
\midrule
GSLM (Unmarg.) & -- & $30{,}763$ & $27{,}687$ & $3{,}076$ \\
SpiRit (Unmarg.) & -- & $29{,}295$ & $26{,}366$ & $2{,}930$ \\
Llama-Mimi (Unmarg.) & -- & $29{,}292$ & $26{,}363$ & $2{,}929$ \\
Llama-Mimi-8B (Unmarg.) & -- & $29{,}292$ & $26{,}363$ & $2{,}929$ \\
\addlinespace
\multicolumn{5}{l}{\textit{Frame-level}} \\
\midrule
GSLM (Unmarg.) & -- & $11{,}571$ & $10{,}414$ & $1{,}157$ \\
SpiRit (Unmarg.) & -- & $22{,}262$ & $20{,}036$ & $2{,}226$ \\
Llama-Mimi (Unmarg.) & -- & $31{,}582$ & $28{,}424$ & $3{,}158$ \\
Llama-Mimi-8B (Unmarg.) & -- & $31{,}582$ & $28{,}424$ & $3{,}158$ \\
\bottomrule
\end{tabular}
\caption{Per-GAMM data summary. \emph{Coverage} is reported only for the marginalized word estimators, and refers to the number of corpus tokens for which the estimator produced an estimate. It is omitted (--) elsewhere, where every token is scored by construction. $N$ obs.\ is the number of pupil samples entering the model, with the mean per-fold train/test split under 10-fold CV (at $\omega=1000$ ms). For each speech LM, $N$ obs.\ also reflects that the word with the largest marginalized surprisal estimate is excluded as an outlier, as the most extreme estimates are most affected by error in importance sampling (see \cref{app:est_var}). For Llama-Mimi, for example, the largest estimate was 484 nats versus 266 nats at the next-largest value. Marg.\ and Unmarg.\ word-level models of a given speech LM are fit on the same folds, so their $N$ and fold sizes are identical.}
\label{tab:data_summary}
\end{table}

\section{Supplementary Results}
\label{app:sup_res}
\onecolumn

\subsection{GPT-2 Results}
\label{app:gpt2}

\begin{figure*}[h]
    \centering
    \includegraphics[width=\linewidth]{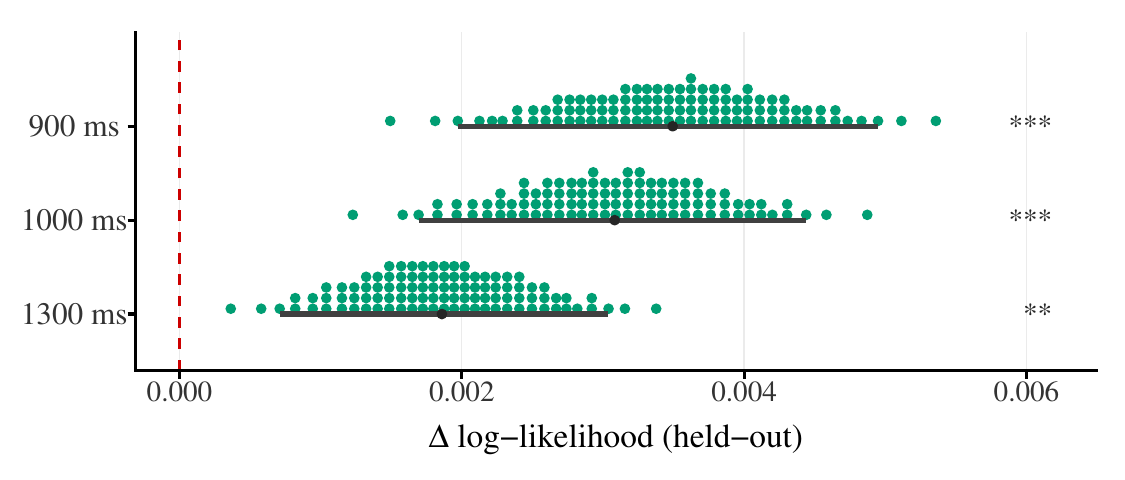}
    \caption{Held-out delta log-likelihood predicting change in pupil size at three offset values $\omega$. Models are fit with word-level surprisal values obtained from GPT-2 using transcriptions of the speech heard by participants.
    Stars correspond to significant improvements over the no-surprisal baseline; corrected $p$-values at $\alpha$ levels of 0.05 (*), 0.01 (**), and 0.001 (***).}
    \label{fig:gpt2}
\end{figure*}

For reference, we also use GPT-2 \citep{radford2019language} to calculate surprisal for each orthographic word in the text transcriptions of our stimuli. We denote tokens from GPT-2 as $\gpttoken \in \gpttokAlphabet$ and calculate surprisal as follows: 

\begin{equation}
    \label{eq:surprisal_gpt}
    \surpgpt\left(\gpttoken_t, \gpttokseq_{<t}\right) = -\log\gptLMPref\left(\gpttoken_t \mid \gpttokseq_{<t}\right).
\end{equation}

\noindent Because the tokenizer of GPT-2 sometimes splits orthographic words across tokens, we sum the surprisals of any multi-token orthographic words. Consistent with other surprisal theory literature, results indicate that adding GPT-2 surprisals to models of pupil size increases the delta log-likelihood significantly over the baseline model (see \cref{fig:gpt2}).

\subsection{Tables of Values}

\begin{table}[h]
\centering
\small
\setlength{\tabcolsep}{6pt}
\begin{tabular}{lccc}
\toprule
Condition & 900 ms & 1000 ms & 1300 ms \\
\midrule
GPT-2 & \makecell{$0.0035$ \\ $[0.0020,\,0.0049]$ \\ ***} & \makecell{$0.0031$ \\ $[0.0017,\,0.0044]$ \\ ***} & \makecell{$0.0019$ \\ $[0.0007,\,0.0030]$ \\ **} \\
\addlinespace
GSLM Marg. & \makecell{$0.0045$ \\ $[0.0031,\,0.0060]$ \\ **} & \makecell{$0.0041$ \\ $[0.0028,\,0.0054]$ \\ **} & \makecell{$0.0021$ \\ $[0.0011,\,0.0030]$ \\ **} \\
\addlinespace
GSLM Unmarg. & \makecell{$-0.0005$ \\ $[-0.0016,\,0.0006]$ \\ n.s.} & \makecell{$-0.0004$ \\ $[-0.0014,\,0.0006]$ \\ n.s.} & \makecell{$-0.0003$ \\ $[-0.0010,\,0.0004]$ \\ n.s.} \\
\addlinespace
SpiRit Marg. & \makecell{$0.0024$ \\ $[0.0009,\,0.0038]$ \\ **} & \makecell{$0.0018$ \\ $[0.0006,\,0.0030]$ \\ *} & \makecell{$0.0015$ \\ $[0.0005,\,0.0025]$ \\ **} \\
\addlinespace
SpiRit Unmarg. & \makecell{$-0.0034$ \\ $[-0.0050,\,-0.0019]$ \\ n.s.} & \makecell{$-0.0031$ \\ $[-0.0044,\,-0.0018]$ \\ n.s.} & \makecell{$-0.0029$ \\ $[-0.0040,\,-0.0018]$ \\ n.s.} \\
\addlinespace
Llama-Mimi Marg. & \makecell{$0.0021$ \\ $[0.0007,\,0.0035]$ \\ **} & \makecell{$0.0018$ \\ $[0.0004,\,0.0031]$ \\ *} & \makecell{$0.0014$ \\ $[0.0003,\,0.0025]$ \\ *} \\
\addlinespace
Llama-Mimi Unmarg. & \makecell{$0.0006$ \\ $[-0.0001,\,0.0013]$ \\ n.s.} & \makecell{$0.0006$ \\ $[-0.0001,\,0.0012]$ \\ n.s.} & \makecell{$0.0001$ \\ $[-0.0004,\,0.0005]$ \\ n.s.} \\
\addlinespace
Llama-Mimi-8B Marg. & \makecell{$0.0032$ \\ $[0.0019,\,0.0045]$ \\ **} & \makecell{$0.0029$ \\ $[0.0016,\,0.0041]$ \\ **} & \makecell{$0.0024$ \\ $[0.0014,\,0.0035]$ \\ **} \\
\addlinespace
Llama-Mimi-8B Unmarg. & \makecell{$0.0010$ \\ $[0.0003,\,0.0017]$ \\ *} & \makecell{$0.0009$ \\ $[0.0003,\,0.0016]$ \\ *} & \makecell{$0.0006$ \\ $[0.0001,\,0.0010]$ \\ *} \\
\addlinespace
\bottomrule
\end{tabular}
\caption{Mean and 95\% confidence intervals for delta log-likelihood predicting change in pupil size at three offset values $\omega$. Models are fit with word-level surprisal values obtained from four speech LMs and GPT-2. 
Stars correspond to significant improvements over the no-surprisal baseline; corrected $p$-values at $\alpha$ levels of 0.05 (*), 0.01 (**), and 0.001 (***).}\label{tab:delta_ll_word}
\end{table}

\begin{table}[h]
\centering
\small
\setlength{\tabcolsep}{6pt}
\begin{tabular}{lccc}
\toprule
Condition & 900 ms & 1000 ms & 1300 ms \\
\midrule
GSLM Unmarg. & \makecell{$0.0018$ \\ $[0.0011,\,0.0025]$ \\ **} & \makecell{$0.0016$ \\ $[0.0010,\,0.0023]$ \\ **} & \makecell{$0.0011$ \\ $[0.0007,\,0.0016]$ \\ **} \\
\addlinespace
SpiRit Unmarg. & \makecell{$0.0025$ \\ $[0.0017,\,0.0033]$ \\ **} & \makecell{$0.0024$ \\ $[0.0016,\,0.0031]$ \\ **} & \makecell{$0.0017$ \\ $[0.0010,\,0.0023]$ \\ **} \\
\addlinespace
Llama-Mimi Unmarg. & \makecell{$0.0021$ \\ $[0.0013,\,0.0030]$ \\ **} & \makecell{$0.0018$ \\ $[0.0011,\,0.0026]$ \\ **} & \makecell{$0.0012$ \\ $[0.0006,\,0.0018]$ \\ **} \\
\addlinespace
Llama-Mimi-8B Unmarg. & \makecell{$0.0023$ \\ $[0.0014,\,0.0032]$ \\ **} & \makecell{$0.0020$ \\ $[0.0012,\,0.0028]$ \\ **} & \makecell{$0.0012$ \\ $[0.0006,\,0.0018]$ \\ **} \\
\addlinespace
\bottomrule
\end{tabular}
\caption{Mean and 95\% confidence intervals for delta log-likelihood predicting change in pupil size at three offset values $\omega$. Models are fit with phone-level surprisal values obtained from four speech LMs. 
Stars correspond to significant improvements over the no-surprisal baseline; corrected $p$-values at $\alpha$ levels of 0.05 (*), 0.01 (**), and 0.001 (***).}\label{tab:delta_ll_phone}
\end{table}

\begin{table}[h]
\centering
\small
\setlength{\tabcolsep}{6pt}
\begin{tabular}{lccc}
\toprule
Condition & 900 ms & 1000 ms & 1300 ms \\
\midrule
GSLM Unmarg. & \makecell{$0.0019$ \\ $[0.0007,\,0.0031]$ \\ **} & \makecell{$0.0009$ \\ $[0.0001,\,0.0017]$ \\ n.s.} & \makecell{$0.0010$ \\ $[0.0001,\,0.0018]$ \\ n.s.} \\
\addlinespace
SpiRit Unmarg. & \makecell{$0.0027$ \\ $[0.0018,\,0.0037]$ \\ **} & \makecell{$0.0027$ \\ $[0.0017,\,0.0036]$ \\ **} & \makecell{$0.0014$ \\ $[0.0007,\,0.0021]$ \\ **} \\
\addlinespace
Llama-Mimi Unmarg. & \makecell{$0.0035$ \\ $[0.0025,\,0.0044]$ \\ **} & \makecell{$0.0030$ \\ $[0.0021,\,0.0039]$ \\ **} & \makecell{$0.0018$ \\ $[0.0010,\,0.0025]$ \\ **} \\
\addlinespace
Llama-Mimi-8B Unmarg. & \makecell{$0.0023$ \\ $[0.0015,\,0.0032]$ \\ **} & \makecell{$0.0020$ \\ $[0.0011,\,0.0028]$ \\ **} & \makecell{$0.0011$ \\ $[0.0004,\,0.0018]$ \\ *} \\
\addlinespace
\bottomrule
\end{tabular}
\caption{Mean and 95\% confidence intervals for delta log-likelihood predicting change in pupil size at three offset values $\omega$. Models are fit with frame-level surprisal values obtained from four speech LMs.
Stars correspond to significant improvements over the no-surprisal baseline; corrected $p$-values at $\alpha$ levels of 0.05 (*), 0.01 (**), and 0.001 (***).}\label{tab:delta_ll_frame}
\end{table}

\subsection{Marginalized and Unmarginalized Comparisons}

\begin{table}[h]
\centering
\small
\begin{tabular}{lccc}
\toprule
Contrast & 900 ms & 1000 ms & 1300 ms \\
\midrule
GSLM & \makecell{$0.0050$ \\ $[0.0032,\,0.0068]$ \\ $p<0.001$} & \makecell{$0.0045$ \\ $[0.0029,\,0.0061]$ \\ $p<0.001$} & \makecell{$0.0024$ \\ $[0.0011,\,0.0036]$ \\ $p<0.001$} \\
\addlinespace
Llama-Mimi & \makecell{$0.0015$ \\ $[0.0001,\,0.0030]$ \\ $p=0.019$} & \makecell{$0.0012$ \\ $[-0.0002,\,0.0026]$ \\ $p=0.044$} & \makecell{$0.0014$ \\ $[0.0003,\,0.0025]$ \\ $p=0.007$} \\
\addlinespace
Llama-Mimi-8B & \makecell{$0.0022$ \\ $[0.0008,\,0.0035]$ \\ $p=0.001$} & \makecell{$0.0019$ \\ $[0.0006,\,0.0032]$ \\ $p=0.002$} & \makecell{$0.0019$ \\ $[0.0008,\,0.0029]$ \\ $p<0.001$} \\
\addlinespace
SpiRit & \makecell{$0.0058$ \\ $[0.0038,\,0.0078]$ \\ $p<0.001$} & \makecell{$0.0049$ \\ $[0.0032,\,0.0066]$ \\ $p<0.001$} & \makecell{$0.0044$ \\ $[0.0031,\,0.0058]$ \\ $p<0.001$} \\
\addlinespace
\bottomrule
\end{tabular}
\caption{Paired-bootstrap tests for differences between marginalized and unmarginalized word surprisal GAMMs by model and time offset. Values represent the average difference in delta log-likelihood between GAMMs fit with marginalized and unmarginalized word surprisals derived from the same speech LMs. Positive values indicate higher delta log-likelihoods when fit with marginalized surprisals. Confidence intervals are 95\% bootstrap intervals.}
\label{tab:marg_vs_unmarg}
\end{table}

\subsection{Pupil Time Baseline}
\label{ap:pup_time}

\begin{figure*}[h]
    \centering
    \includegraphics[width=\linewidth]{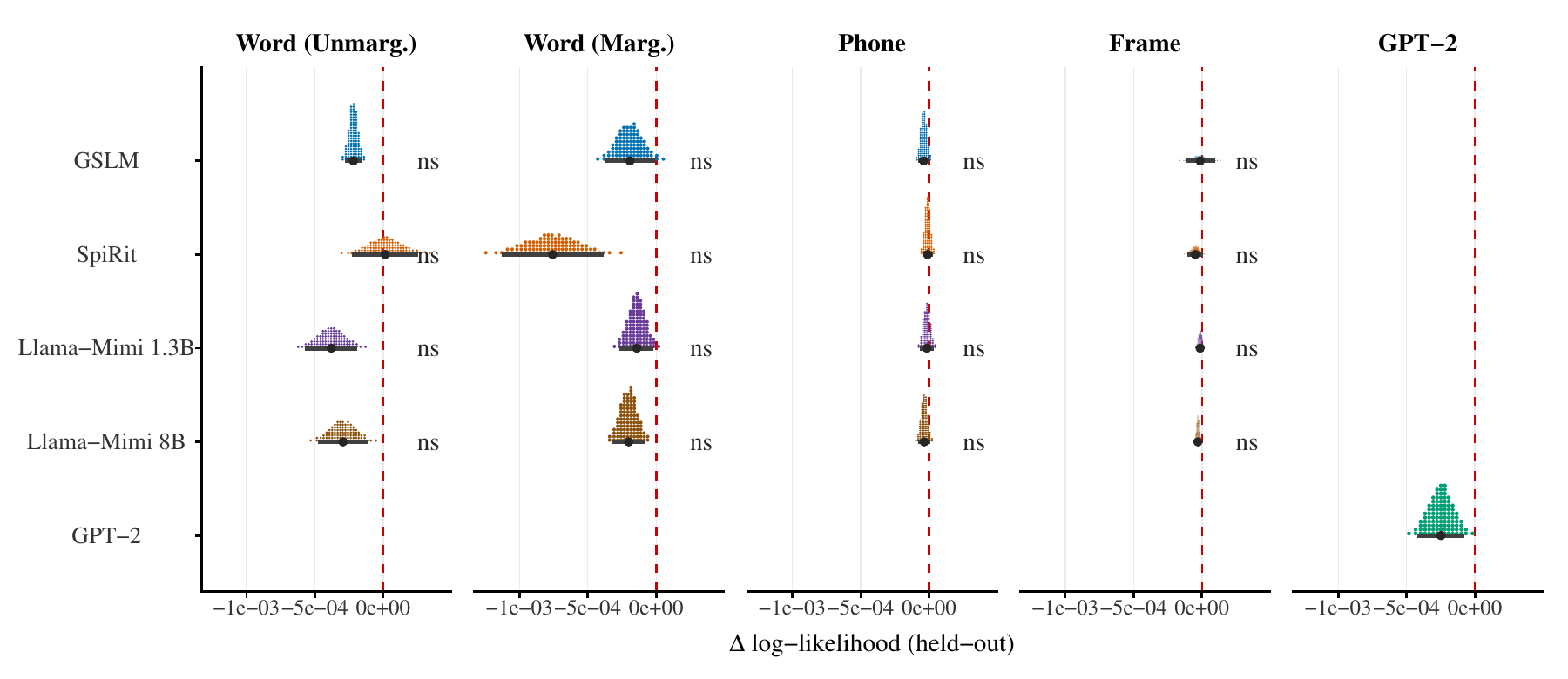}
    \caption{Held-out delta log-likelihood relative to baseline model of pupil size, augmented with a smooth for the time in the trial at which the pupil measurement was taken (offset: $\omega=1000$ ms).
    Each dot represents 1\% of the bootstrap distribution of the mean delta log-likelihood; the point and bars below give the bootstrap mean and 95\% CI.}
    \label{fig:full_dll_timesmooth}
\end{figure*}

\end{document}